\documentclass[11pt, a4paper, 
onecolumn, 
showabstract]{naverlabseurope-onecol-flex}

\setcitestyle{authoryear,round}
\usepackage[utf8]{inputenc}
\usepackage[T1]{fontenc}
\usepackage{amsmath,amssymb,amsthm}
\usepackage{hyperref}
\usepackage{enumitem}
\usepackage{tikz}
\usetikzlibrary{arrows.meta,decorations.markings}

\newtheorem{theorem}{Theorem}
\newtheorem{proposition}[theorem]{Proposition}
\newtheorem{corollary}[theorem]{Corollary}

\theoremstyle{definition}

\newtheorem{example}[theorem]{Example}
\theoremstyle{remark}
\newtheorem{remark}[theorem]{Remark}

\newcommand{\KL}[2]{\mathrm{KL}\!\left(#1 \,\|\, #2\right)}
\newcommand{\E}{\mathbb{E}}
\newcommand{\R}{\mathbb{R}}
\newcommand{\cM}{\mathcal{M}}
\newcommand{\cF}{\mathcal{F}}

\authors{Marc Dymetman\textsuperscript{*}}
\affiliations{$^{*}$Scientific Consultant}
\website{\texttt{marc.dymetman@gmail.com}}
\websiteref{}

\title{Exponential families from a single KL identity}
\titlerunning{Exponential families from a single KL identity}

\begin{abstract}
Exponential families encompass the distributions central to modern machine learning --- softmax, Gaussians, and Boltzmann distributions --- and underlie the theory of variational inference, entropy-regularized reinforcement learning, and RLHF. We isolate a simple identity for exponential families that expresses the KL difference $\mathrm{KL}(q \| p_{\lambda_2}) - \mathrm{KL}(q \| p_{\lambda_1})$ in terms of the log-partition function $A(\lambda)$ and the moment $\mu_q$. Remarkably, this identity together with the single fact that $\mathrm{KL} \geq 0$ (with equality iff $p = q$) suffices, by direct substitution and rearrangement, to derive a cluster of results that are classically obtained by separate, heavier arguments: a generalized three-point identity for arbitrary reference distributions, Pythagorean theorems for I-projections and reverse I-projections, convexity of the log-partition function, identification of its Legendre dual in KL terms, the Gibbs variational principle, and the explicit optimizer in KL-regularized reward maximization, including the exponential tilting formula underlying entropy-regularized control and RLHF. Beyond these purely algebraic consequences, standard analytic arguments recover the gradient formula for the log-partition function, the Bregman representation of within-family KL divergence, and the surjectivity of the moment map. The note is self-contained.
\end{abstract}

\begin{document}

\maketitle

\section{Setup}

Let $Y$ be a finite or countable set, and let $a$ be a strictly positive probability distribution on $Y$ (the \emph{base distribution}).
Let $\phi \colon Y \to \R^d$ be a function (the \emph{sufficient statistic}).
For $\lambda \in \R^d$, define the \emph{partition function}
\[
    Z_\lambda = \sum_{y \in Y} a(y)\, e^{\lambda \cdot \phi(y)},
\]
the \emph{log-partition function} $A(\lambda) = \log Z_\lambda$,
and the \emph{natural parameter space} $\Lambda = \{\lambda \in \R^d : Z_\lambda < \infty\}$.
When $Y$ is finite, $\Lambda = \R^d$.
For $\lambda \in \Lambda$, the \emph{exponential family member} with natural parameter $\lambda$ is
\[
    p_\lambda(y) = \frac{a(y)\,e^{\lambda \cdot \phi(y)}}{Z_\lambda}.
\]
Note that $p_0 = a$ and $A(0) = 0$.
Since $a(y) > 0$ for all $y \in Y$, we have $p_\lambda(y) > 0$ for all $y$ and $\lambda \in \Lambda$.
For any distribution $q$ on $Y$ such that $\sum_{y} q(y)|\phi_j(y)| < \infty$ for all $j$, we write
\[
    \mu_q = \sum_{y \in Y} q(y)\,\phi(y) \in \R^d.
\]
When $q = p_\lambda$, we write $\mu_\lambda = \mu_{p_\lambda}$; this is always finite when $Y$ is finite.
We refer to \cite{Brown1986,BarndorffNielsen1978,WainwrightJordan2008} for classical treatments of exponential families.

\section{The identity}\label{sec:identity}

\begin{proposition}[KL difference identity]\label{prop:identity}
Let $\lambda_1, \lambda_2 \in \Lambda$ and let $q$ be a distribution on $Y$ such that $\mu_q$ exists.
If $Y$ is finite, then
\begin{equation}\label{eq:identity}
    \boxed{\KL{q}{p_{\lambda_2}} - \KL{q}{p_{\lambda_1}}
    = A(\lambda_2) - A(\lambda_1) + \mu_q \cdot (\lambda_1 - \lambda_2)}
\end{equation}
holds unconditionally.
If $Y$ is countably infinite, it holds whenever at least one of $\KL{q}{p_{\lambda_1}}$, $\KL{q}{p_{\lambda_2}}$ is finite; in that case both are finite.
\end{proposition}

\begin{proof}
Since $p_\lambda(y) > 0$ for all $y$, the log-ratio of two exponential family members is well-defined and affine in $\phi$:
\[
    \log \frac{p_{\lambda_1}(y)}{p_{\lambda_2}(y)}
    = (\lambda_1 - \lambda_2) \cdot \phi(y) + A(\lambda_2) - A(\lambda_1).
\]
Taking the expectation under $q$ (finite since $\mu_q$ exists):
\begin{equation}\label{eq:logratio}
    \E_q\!\left[\log \frac{p_{\lambda_1}}{p_{\lambda_2}}\right]
    = (\lambda_1 - \lambda_2) \cdot \mu_q + A(\lambda_2) - A(\lambda_1).
\end{equation}
The identity then follows from
\[
    \KL{q}{p_{\lambda_2}} - \KL{q}{p_{\lambda_1}}
    = \E_q\!\left[\log \frac{q}{p_{\lambda_2}} - \log \frac{q}{p_{\lambda_1}}\right]
    = \E_q\!\left[\log \frac{p_{\lambda_1}}{p_{\lambda_2}}\right],
\]
where the difference on the left is well-defined (and both KLs are finite) since their difference equals the finite quantity~\eqref{eq:logratio}.
\end{proof}

Part of the interest of identity~\eqref{eq:identity} is that it relates three different types of quantities in a single linear equation: KL divergences (the left side), log-partition values (the $A(\lambda)$ terms), and moments (the $\mu_q$ term).
Different specializations move between these worlds: eliminating the log-partition terms yields the three-point identity (Section~\ref{sec:algebraic}), while eliminating the KL terms yields convexity of $A(\lambda)$ (Section~\ref{sec:logpartition}).

\begin{example}[Categorical distributions and softmax]\label{ex:categorical}
Let $Y = \{1,\ldots,k\}$ with uniform base measure $a(y) = 1/k$, and let $\phi(y) = e_y \in \R^k$ be the $y$-th standard basis vector, so $\lambda \cdot \phi(y) = \lambda_y$.
The log-partition function is
\[
    A(\lambda)
    = \log \tfrac{1}{k}\textstyle\sum_{y=1}^k e^{\lambda_y}
    = \mathrm{LSE}(\lambda) - \log k,
\]
where $\mathrm{LSE}(\lambda) = \log\sum_y e^{\lambda_y}$ is the log-sum-exp function, and the exponential family member is the softmax distribution:
\[
    p_\lambda(y) = \frac{e^{\lambda_y}}{\sum_{j=1}^k e^{\lambda_j}} = \mathrm{softmax}(\lambda)_y.
\]
Since $\phi(y) = e_y$, for any distribution $q$ on $Y$ the moment is $\mu_q = \E_q[e_Y] = q$ (the probability vector itself), and in particular $\mu_\lambda = p_\lambda$.

\smallskip\noindent\textit{The identity.}
Using $A(\lambda_2) - A(\lambda_1) = \mathrm{LSE}(\lambda_2) - \mathrm{LSE}(\lambda_1)$, identity~\eqref{eq:identity} becomes
\[
    \KL{q}{p_{\lambda_2}} - \KL{q}{p_{\lambda_1}}
    = \mathrm{LSE}(\lambda_2) - \mathrm{LSE}(\lambda_1)
      + \sum_{y=1}^k q_y\,(\lambda_{1,y} - \lambda_{2,y}).
\]
This can be verified directly: since $\log p_\lambda(y) = \lambda_y - \mathrm{LSE}(\lambda)$,
\[
    \KL{q}{p_\lambda} = \mathrm{LSE}(\lambda) - \textstyle\sum_y q_y\,\lambda_y - H(q),
\]
where $H(q) = -\sum_y q_y \log q_y$, and the left-hand difference reduces immediately to the right-hand side.

\smallskip\noindent\textit{KL-regularized optimization.}
For a reward $r \colon Y \to \R$ and temperature $\beta > 0$, Corollary~\ref{cor:rl} gives
\[
    \max_q \bigl\{\E_q[r] - \beta\,\KL{q}{a}\bigr\}
    = \beta\,\mathrm{LSE}(r/\beta) - \beta\log k,
\]
with unique maximizer $q^*(y) = \mathrm{softmax}(r/\beta)_y \propto e^{r(y)/\beta}$.
Since $\KL{q}{a} = \log k - H(q)$ for uniform $a$, the objective is equivalently
$\E_q[r] + \beta\,H(q)$, the standard maximum-entropy objective in reinforcement
learning \cite{Ziebart2010,Haarnoja2018}.
The optimal policy is the Boltzmann distribution at temperature $\beta$, and the
maximum of this entropy-regularized objective is $\beta\,\mathrm{LSE}(r/\beta)$,
the soft-maximum of $r$.
\end{example}

\subsection*{Extension to general measurable spaces}

For readers interested in the fully general setting, the identity extends to arbitrary measurable spaces with no change to the algebra.
Let $(Y, \cF, a)$ be a probability space and $\phi \colon Y \to \R^d$ a measurable function.
The partition function, natural parameter space, and exponential family are defined as before with sums replaced by integrals:
\[
    Z_\lambda = \int_Y e^{\lambda \cdot \phi(y)}\, da(y), \quad
    p_\lambda(y) = \frac{e^{\lambda \cdot \phi(y)}}{Z_\lambda}
\]
(density with respect to $a$), and $\mu_q = \E_q[\phi]$ when $\E_q[|\phi_j|] < \infty$ for all $j$.
The identity~\eqref{eq:identity} holds under the same conditions as Proposition~\ref{prop:identity}, with the countably infinite case replaced by the requirement that at least one of $\KL{q}{p_{\lambda_1}}$, $\KL{q}{p_{\lambda_2}}$ is finite.
The key additional step in the proof is to establish absolute continuity: since all $p_\lambda$ have strictly positive densities with respect to $a$, $\KL{q}{p_{\lambda_1}} < \infty$ implies $q \ll p_{\lambda_1} \sim a \sim p_{\lambda_2}$, so $\log(q/p_{\lambda_2})$ is $q$-integrable and $\KL{q}{p_{\lambda_2}} < \infty$ follows by linearity of expectation.
The integrability of $\mu_q$ under $p_\lambda$ for $\lambda \in \mathrm{int}(\Lambda)$ is guaranteed by a standard exponential moment argument: since $Z_{\lambda \pm te_j} < \infty$ for small $t > 0$, the inequality $|\phi_j|\,e^{\lambda\cdot\phi} \leq t^{-1} e^{(\lambda \pm te_j)\cdot\phi}$ gives integrability.

\section{Algebraic consequences}\label{sec:algebraic}

\textit{All results in this section follow from the identity~\eqref{eq:identity} and the single fact that $\KL{p}{q} \geq 0$ with equality if and only if $p = q$.}
Remarkably, KL nonnegativity---together with the equality condition $\KL{p}{q} = 0 \iff p = q$---is the sole ingredient behind all minimization and projection results in this note.
No appeal to analysis, differentiation, or convexity theory is needed for any of these results.
Each result is obtained by choosing specific values of $q$, $\lambda_1$, $\lambda_2$ and rearranging.

\subsection{Generalized three-point identity}

Setting $q = p_{\lambda_1}$ in~\eqref{eq:identity} (which requires $\mu_{\lambda_1}$ to exist, e.g., $\lambda_1 \in \mathrm{int}(\Lambda)$) gives a \emph{specialized} form of the identity:
\begin{equation}\label{eq:specialized}
    \KL{p_{\lambda_1}}{p_{\lambda_2}}
    = A(\lambda_2) - A(\lambda_1) + \mu_{\lambda_1} \cdot (\lambda_1 - \lambda_2).
\end{equation}
Using this to eliminate the log-partition difference in the general identity~\eqref{eq:identity} yields a three-point identity.

\begin{corollary}[Generalized three-point identity]\label{cor:threepoint}
Under the conditions of Proposition~\ref{prop:identity}, and assuming $\mu_{\lambda_1}$ exists,
\begin{equation}\label{eq:threepoint}
    \KL{q}{p_{\lambda_2}}
    = \KL{q}{p_{\lambda_1}} + \KL{p_{\lambda_1}}{p_{\lambda_2}}
      + (\mu_q - \mu_{\lambda_1}) \cdot (\lambda_1 - \lambda_2).
\end{equation}
\end{corollary}

\begin{proof}
Subtract~\eqref{eq:specialized} from~\eqref{eq:identity} (with the same $\lambda_1, \lambda_2$):
the log-partition terms cancel, leaving~\eqref{eq:threepoint}.
\end{proof}

The last term is an inner product between the moment mismatch $\mu_q - \mu_{\lambda_1}$ and the parameter difference $\lambda_1 - \lambda_2$.

\begin{remark}
The case where $q$ is an arbitrary distribution not necessarily in the exponential family does not appear to be readily found in the literature.
It is this generality that makes the subsequent projection results directly applicable to arbitrary distributions: the Pythagorean theorem (Corollary~\ref{cor:pythagoras}) characterizes the reverse I-projection of any $q$ onto the family, and the I-projection corollary (Corollary~\ref{cor:maxent}) characterizes the I-projection of $a$ onto any moment slice.
\end{remark}

\subsection{Pythagorean theorem and reverse I-projection}\label{sec:pythagoras}

When the inner product in~\eqref{eq:threepoint} vanishes, one obtains a Pythagorean theorem.

\begin{corollary}[Pythagorean theorem]\label{cor:pythagoras}
If\, $(\mu_q - \mu_{\lambda_1}) \cdot (\lambda_1 - \lambda_2) = 0$, then
\[
    \KL{q}{p_{\lambda_2}} = \KL{q}{p_{\lambda_1}} + \KL{p_{\lambda_1}}{p_{\lambda_2}}.
\]
In particular, if $p_{\lambda_1}$ is the moment-matching member of the family for $q$ (i.e., $\mu_{\lambda_1} = \mu_q$), then this holds for all $\lambda_2 \in \Lambda$, and consequently $p_{\lambda_1}$ is the unique closest member of the family to $q$ in KL divergence (the \emph{reverse I-projection}, or \emph{reverse information projection}, of $q$ onto the family, in the terminology of \cite{CsiszarMatus2003}).
\end{corollary}

\begin{proof}
When $\mu_{\lambda_1} = \mu_q$, the inner product vanishes for every $\lambda_2$, and~\eqref{eq:threepoint} reduces to the stated identity.
Since $\KL{p_{\lambda_1}}{p_{\lambda_2}} \geq 0$, this gives $\KL{q}{p_{\lambda_2}} \geq \KL{q}{p_{\lambda_1}}$ for all $\lambda_2$, with equality iff $p_{\lambda_2} = p_{\lambda_1}$.
\end{proof}

\begin{remark}
The existence of a moment-matching $\lambda_1$ (i.e., $\mu_{\lambda_1} = \mu_q$) is a separate question, addressed in Section~\ref{sec:surjectivity}. The above result says: \emph{if} such a $\lambda_1$ exists, \emph{then} it yields the unique reverse I-projection.
\end{remark}

\subsection{I-projection onto moment slices}\label{sec:maxent}

The identity also characterizes $p_\lambda$ as the I-projection of the base distribution $a$ onto the moment slice $\cM_\mu$.
Equivalently, $p_\lambda$ is the maximum-entropy distribution subject to the moment constraint, without Lagrange multipliers.
Given $\mu \in \R^d$, define the \emph{moment slice}
\[
    \cM_\mu = \{q : \E_q[\phi] = \mu\},
\]
i.e., the set of all distributions with mean parameter~$\mu$.

\begin{corollary}[I-projection onto $\cM_\mu$]\label{cor:maxent}
If $\lambda \in \Lambda$ satisfies $\mu_\lambda = \mu$, then $p_\lambda$ is the unique minimizer of $\KL{q}{a}$ over $q \in \cM_\mu$ (equivalently, the unique maximizer of entropy relative to the distribution $a$ subject to the moment constraint).
\end{corollary}

\begin{proof}
Setting $\lambda_2 = 0$ and $\lambda_1 = \lambda$ in~\eqref{eq:threepoint}, and using $p_0 = a$:
\[
    \KL{q}{a}
    = \KL{q}{p_\lambda} + \KL{p_\lambda}{a}
      + (\mu_q - \mu_\lambda) \cdot (\lambda - 0).
\]
For $q \in \cM_\mu$ with $\mu_\lambda = \mu$, we have $\mu_q = \mu_\lambda$, so the last term vanishes:
\[
    \KL{q}{a} = \KL{q}{p_\lambda} + \KL{p_\lambda}{a}.
\]
Since $\KL{q}{p_\lambda} \geq 0$ with equality iff $q = p_\lambda$, the minimum of $\KL{q}{a}$ over $\cM_\mu$ is uniquely attained at $q = p_\lambda$.
\end{proof}

\begin{remark}
This says that $p_\lambda$ is the I-projection of the base distribution $a$ onto the moment slice $\cM_\mu$.
As in Corollary~\ref{cor:pythagoras}, the existence of a $\lambda$ with $\mu_\lambda = \mu$ is addressed separately in Section~\ref{sec:surjectivity}.
\end{remark}

\begin{figure}[h]
\centering
\begin{tikzpicture}[scale=1.05, >=Stealth]
\fill[gray!20]
    (0.8, 3.2)
    -- (2.4, 1.5)
    -- (8.0, 1.5)
    -- (6.4, 3.2)
    -- cycle;
  \node at (2.8, 1.9) {$\mathcal{M}_\mu$};
  \draw[thick,
    postaction={decorate,
      decoration={markings,
        mark=at position 0.38 with {\arrowreversed[scale=1.4]{Stealth}}
      }
    }
  ]
    (7.5, 0.4)
    .. controls (7.1, 1.2) and (6.8, 1.8) ..
    (6.5, 2.4)
    .. controls (6.1, 3.1) and (5.5, 3.9) ..
    (5.1, 4.8)
    .. controls (4.8, 5.5) and (5.2, 6.2) ..
    (6.1, 7.0);
  \node at (6.6, 6.6) {$\mathcal{E}$};
  \node[right] at (6.0, 3.4) {\small I-proj};
  \filldraw (6.5, 2.4) circle (2.2pt);
  \node[right] at (6.6, 2.35) {$p_\lambda$};
  \filldraw (5.1, 4.8) circle (2.2pt);
  \node[left] at (5.0, 4.8) {$p_0 = a$};
  \filldraw (3.9, 2.25) circle (2.2pt);
  \node[below] at (3.9, 2.15) {$q$};
  \draw[-{Stealth[scale=1.4]}, dashed, thick] (3.9, 2.25) -- (6.38, 2.38);
  \node[above, rotate=3] at (5.0, 2.3) {\small rev.\ I-proj};
  \draw (6.367, 2.647) -- (6.087, 2.647) -- (6.22, 2.4);
\end{tikzpicture}
\caption{The reverse I-projection of $q \in \mathcal{M}_\mu$ onto the exponential
family $\mathcal{E}$ and the I-projection of $a = p_0$ onto the moment slice
$\mathcal{M}_\mu$ both yield $p_\lambda$. The right angle at $p_\lambda$ reflects
Corollary~\ref{cor:pythagoras}.}
\label{fig:projections}
\end{figure}
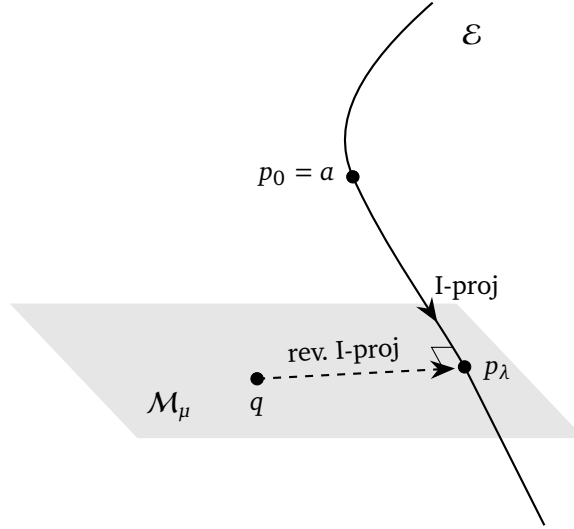

\subsection{Gibbs variational principle (ELBO)}\label{sec:elbo}

\begin{corollary}\label{cor:gibbs}
For all $\lambda \in \Lambda$ and all distributions $q$ with $\E_q[|\phi_j|] < \infty$ for all $j$ and $\KL{q}{a} < \infty$,
\begin{equation}\label{eq:elbo}
    A(\lambda) = \underbrace{\lambda \cdot \mu_q - \KL{q}{a}}_{\mathrm{ELBO}(q,\lambda)} + \KL{q}{p_\lambda}.
\end{equation}
In particular,
\[
    A(\lambda) = \sup_q \left\{ \lambda \cdot \mu_q - \KL{q}{a} \right\},
\]
and the supremum is attained uniquely at $q = p_\lambda$.
\end{corollary}

\begin{proof}
Set $\lambda_1 = 0$ in~\eqref{eq:identity}, so $p_{\lambda_1} = a$ and $A(0) = 0$.
For $\lambda_2 = \lambda$:
\[
    \KL{q}{p_\lambda} = \KL{q}{a} + A(\lambda) - \lambda \cdot \mu_q.
\]
Rearranging gives~\eqref{eq:elbo}, and the variational characterization follows from $\KL{q}{p_\lambda} \geq 0$ with equality iff $q = p_\lambda$.
\end{proof}

\begin{remark}
The decomposition~\eqref{eq:elbo} is widely known in machine learning as the \emph{ELBO (Evidence Lower Bound)} decomposition \cite{Jordan1999,BleiKucukelbirMcAuliffe2017}: the log-evidence $A(\lambda)$ equals the ELBO plus the KL gap.
In statistical physics, the same result appears as the \emph{variational free energy} principle, with $-A(\lambda)$ as free energy, $\KL{q}{a}$ as negative entropy, and $\lambda \cdot \mu_q$ as expected energy.
In probability and large deviations, it is the \emph{Gibbs variational principle}; 
see \cite[Chapter~4]{PolyanskiyWu2025} for a modern treatment and historical references.
The variational characterization $A(\lambda) = \sup_q\{\lambda\cdot\mu_q - \KL{q}{a}\}$ directly yields the solution to KL-regularized reward maximization: setting $d=1$, $\phi(y)=r(y)$, $\lambda=1/\beta$ and multiplying through by $\beta$ gives $\sup_q\{\E_q[r]-\beta\,\KL{q}{a}\} = \beta A(1/\beta)$, which is the content of the next corollary. We present it separately because the explicit Boltzmann form of the optimizer $q^*\propto a\,e^{r/\beta}$, the role of the temperature parameter $\beta$, and the connections to the reinforcement learning and RLHF literature warrant a dedicated statement.
\end{remark}

\subsection{KL-regularized reward maximization}\label{sec:klcontrol}

Corollary~\ref{cor:gibbs} directly yields the solution to KL-regularized reward maximization, a key result in reinforcement learning. Setting $d=1$, $\phi(y)=r(y)$, $\lambda=1/\beta$ in the variational characterization and multiplying by $\beta$ gives the following.

\begin{corollary}\label{cor:rl}
Let $d = 1$, $\phi(y) = r(y)$ a reward function, and $\beta > 0$ a regularization parameter such that $1/\beta \in \Lambda$ (i.e., $\E_a[e^{r/\beta}] < \infty$; this is automatic when $Y$ is finite, or more generally when $r$ is bounded above).
Then
\[
    \max_q \left\{ \E_q[r] - \beta\, \KL{q}{a} \right\}
    = \beta A(1/\beta),
\]
where the maximum is over distributions $q$ with $\E_q[|r|] < \infty$ and $\KL{q}{a} < \infty$,
and the unique maximizer is the exponential family member $q^* = p_{1/\beta}$, i.e., $q^*(y) \propto a(y)\, e^{r(y)/\beta}$.
\end{corollary}

\begin{proof}
Set $d=1$, $\phi=r$, $\lambda=1/\beta$ in Corollary~\ref{cor:gibbs} and multiply through by $\beta$.
\end{proof}

\begin{remark}
The exponential form of the optimal distribution in KL-regularized optimization has deep roots; 
the distribution $p_{1/\beta}$ is variously called the Boltzmann or Gibbs distribution in 
statistical physics, and the softmax or Boltzmann policy in reinforcement learning.
In control theory, \cite{Todorov2007} introduced linearly solvable MDPs with KL control costs, 
where the optimal policy takes this exponential form; related results were obtained independently 
by \cite{Kappen2005} via path integral methods.
In maximum entropy reinforcement learning, \cite{Ziebart2010} derived the same form for the 
entropy-regularized setting (the special case $a = \mathrm{uniform}$); the soft actor-critic 
framework of \cite{Haarnoja2018} extends these ideas to the deep RL setting.
The LLM fine-tuning literature largely rediscovered these results independently, with the 
exponential form appearing as both an optimal policy and a modeling target.
\cite{Khalifa2021} give an early application that explicitly uses exponential families and moment 
constraints, framing fine-tuning as forward-KL minimization toward a target distribution; 
\cite{Ouyang2022} established the canonical RLHF setting; \cite[Theorem~1]{Korbak2022} provide 
a clean statement of the equivalence between KL-regularized reward maximization and reverse-KL 
minimization toward $p_{1/\beta}$, and \cite{korbak_rl_2022} offer a concurrent Bayesian 
interpretation of the same equivalence; \cite{Rafailov2023} exploit the exponential form to 
derive a closed-form policy optimization objective, connecting RLHF to direct preference 
optimization.
Together, these threads illustrate how the same variational principle has been independently 
rediscovered across communities, each time yielding new algorithmic and conceptual insights.
\end{remark}

\subsection{Convexity of $A$ and the supporting hyperplane property}\label{sec:convexity}

\begin{corollary}\label{cor:subgradient}
Assume that $\mu_\lambda$ exists for every $\lambda\in\Lambda$. Then, for all
$\lambda_1,\lambda_2\in\Lambda$,
\begin{equation}\label{eq:subgradient}
    A(\lambda_2) \geq A(\lambda_1) + \mu_{\lambda_1} \cdot (\lambda_2 - \lambda_1).
\end{equation}
In other words, $A$ is convex on $\Lambda$, and the hyperplane through $({\lambda_1}, A(\lambda_1))$ with slope $\mu_{\lambda_1}$ is a global supporting hyperplane of $A$.
\end{corollary}

\begin{proof}
The specialized identity~\eqref{eq:specialized} gives
$A(\lambda_2) - A(\lambda_1) - \mu_{\lambda_1}\cdot(\lambda_2 - \lambda_1) = \KL{p_{\lambda_1}}{p_{\lambda_2}} \geq 0$.
\end{proof}

\begin{remark}\label{rem:strict}
If the parametrization is injective ($\lambda_1 \neq \lambda_2$ implies $p_{\lambda_1} \neq p_{\lambda_2}$), then the inequality in~\eqref{eq:subgradient} is strict for $\lambda_1 \neq \lambda_2$, since the gap equals $\KL{p_{\lambda_1}}{p_{\lambda_2}}$, which vanishes only when $p_{\lambda_1} = p_{\lambda_2}$.
Thus $A$ is strictly convex.
\end{remark}

\subsection{The dual function $A^*$ and Legendre duality}\label{sec:dual}

Define the \emph{Legendre dual} of $A$ as
\[
    A^*(\mu) = \sup_{\lambda \in \Lambda} \left\{ \lambda \cdot \mu - A(\lambda) \right\}.
\]
By definition, $A^*$ is convex in $\mu$ (as a supremum of affine functions), and the Legendre--Fenchel inequality $A(\lambda) + A^*(\mu) \geq \lambda \cdot \mu$ holds for all $\lambda, \mu$.

\begin{corollary}[Identification of $A^*$]\label{cor:dual}
If there exists $\lambda \in \Lambda$ with $\mu_\lambda = \mu$, then the supremum in the definition of $A^*(\mu)$ is attained at this $\lambda$, and
\begin{equation}\label{eq:dual}
    A^*(\mu) = \KL{p_\lambda}{a}.
\end{equation}
\end{corollary}

\begin{proof}
From the specialized identity~\eqref{eq:specialized} with $\lambda_2 = 0$ (so $p_{\lambda_2} = a$ and $A(0) = 0$):
\[
    \KL{p_\lambda}{a} = \lambda \cdot \mu_\lambda - A(\lambda) = \lambda \cdot \mu - A(\lambda).
\]
So $\lambda \cdot \mu - A(\lambda) = \KL{p_\lambda}{a}$.
It remains to show this is the supremum over all $\lambda' \in \Lambda$.
For any $\lambda' \in \Lambda$, the supporting hyperplane property~\eqref{eq:subgradient} with $\mu_{\lambda_1} = \mu_\lambda = \mu$ gives:
\[
    A(\lambda') \geq A(\lambda) + \mu \cdot (\lambda' - \lambda),
\]
which rearranges to $\lambda' \cdot \mu - A(\lambda') \leq \lambda \cdot \mu - A(\lambda) = \KL{p_\lambda}{a}$.
Hence $A^*(\mu) = \KL{p_\lambda}{a} \geq 0$, with the nonnegativity following from KL $\geq 0$.
\end{proof}

\begin{remark}
The identification $A^*(\mu) = \KL{p_\lambda}{a}$ gives a probabilistic
interpretation of the Legendre dual: $A^*(\mu)$ is the KL divergence from
the moment-matching exponential family member $p_\lambda$ to the base
distribution $a$.
Together with $A(\lambda) = \lambda \cdot \mu - A^*(\mu)$ (from the proof above), this recovers the Legendre--Fenchel relation $A(\lambda) + A^*(\mu) = \lambda \cdot \mu$ as an identity (not merely an inequality) at the matching pair $(\lambda, \mu)$.
\end{remark}

\section{Analytic complements}\label{sec:logpartition}

The results of Section~\ref{sec:algebraic} required nothing beyond $\KL{\cdot}{\cdot} \geq 0$.
We now collect the results that require one additional analytic input: the differentiability of $A$ on $\mathrm{int}(\Lambda)$, and a geometric argument for the surjectivity of the moment map.

\subsection{Differentiability and the gradient of $A$}\label{sec:gradient}

\begin{remark}\label{rem:gradient}
The log-partition function $A(\lambda)$ is in fact differentiable on $\mathrm{int}(\Lambda)$, with $\nabla A(\lambda) = \mu_\lambda$.
When $Y$ is finite, this is immediate since $A(\lambda) = \log \sum_y a(y)\, e^{\lambda\cdot\phi(y)}$ is a finite sum of smooth functions.
In general, it follows from differentiating under the integral sign, justified by dominated convergence using the fact that $Z_\lambda$ converges in a neighborhood of any $\lambda \in \mathrm{int}(\Lambda)$; see, e.g., \cite[Chapter~2]{Brown1986} or \cite[Proposition~3.1]{WainwrightJordan2008}.
This is the only analytic input in this note that does not follow from the identity.
\end{remark}

\subsection{Bregman divergence interpretation}

Given differentiability, the specialized identity~\eqref{eq:specialized} acquires a classical interpretation.
Recall that for a differentiable convex function $F$, the \emph{Bregman divergence} is defined as $B_F(x,y) = F(x) - F(y) - \nabla F(y) \cdot (x - y)$ \cite{Bregman1967}.
Using $\nabla A(\lambda_1) = \mu_{\lambda_1}$ (Remark~\ref{rem:gradient}), the specialized identity~\eqref{eq:specialized} becomes
\begin{equation}\label{eq:bregman}
    \KL{p_{\lambda_1}}{p_{\lambda_2}}
    = A(\lambda_2) - A(\lambda_1) - \nabla A(\lambda_1) \cdot (\lambda_2 - \lambda_1),
\end{equation}
which is precisely the Bregman divergence of the log-partition function evaluated at $(\lambda_2, \lambda_1)$.
This is the well-known representation of within-family KL divergence as a Bregman divergence \cite{Banerjee2005}.

\begin{remark}\label{rem:bregman_threepoint}
With the Bregman interpretation in hand, the generalized three-point identity (Corollary~\ref{cor:threepoint}) can be compared with the standard three-point property of Bregman divergences \cite{Banerjee2005}; see also \cite{Nielsen2010}.
When $q = p_{\lambda_0}$ for some $\lambda_0 \in \Lambda$, identity~\eqref{eq:threepoint} reduces to the standard Bregman three-point property, with each KL divergence replaced by the corresponding Bregman divergence of $A(\lambda)$ via~\eqref{eq:bregman}:
\[
    \KL{p_{\lambda_0}}{p_{\lambda_2}}
    = \KL{p_{\lambda_0}}{p_{\lambda_1}} + \KL{p_{\lambda_1}}{p_{\lambda_2}}
      + (\mu_{\lambda_0} - \mu_{\lambda_1}) \cdot (\lambda_1 - \lambda_2).
\]
However, the general form~\eqref{eq:threepoint}, where $q$ is an \emph{arbitrary} distribution outside the exponential family, does not seem to be readily found in the literature.
This generality is what makes the Pythagorean theorem (Corollary~\ref{cor:pythagoras}) and the maximum entropy characterization (Corollary~\ref{cor:maxent}) directly applicable to reverse I-projections of arbitrary distributions onto the exponential family.
The Pythagorean theorem for exponential families is classically developed in the framework of information geometry \cite{Amari1985,AmariNagaoka2000}, using the dually flat structure of the manifold.
A related Pythagorean property for I-projections in a more general setting (without explicit exponential family structure) was established by \cite{Csiszar1975}.
\end{remark}

\subsection{Surjectivity of the moment map}\label{sec:surjectivity}

Several results in Section~\ref{sec:algebraic} are conditional on the existence of a
parameter $\lambda$ with $\mu_\lambda=\mu$. In the discrete setting this can be shown
directly.

Assume throughout this subsection that $Y$ is finite or countably infinite, so that
\[
M=\operatorname{conv}(\phi(Y)).
\]

\begin{proposition}\label{prop:surjectivity}
Assume $\Lambda=\R^d$ (which holds automatically when $Y$ is finite, and is an additional
assumption when $Y$ is countably infinite, e.g.\ when $\phi$ is bounded).
For any $\mu\in\mathrm{int}(M)$, there exists $\lambda^*\in\R^d$ such that
$\mu_{\lambda^*}=\mu$.
\end{proposition}

\begin{proof}
Set $f(\lambda)=A(\lambda)-\lambda\cdot\mu$.
By Remark~\ref{rem:gradient}, $\nabla f(\lambda)=\mu_\lambda-\mu$,
so it suffices to show that $f$ attains its minimum.

\medskip\noindent\emph{Growth at infinity.}
Since $\mu\in\mathrm{int}(M)\subset\operatorname{conv}(\phi(Y))$, by Carath\'eodory's theorem there
exist finitely many points $y_1,\dots,y_m\in Y$ such that
$\mu\in\mathrm{int}(\operatorname{conv}\{\phi(y_1),\dots,\phi(y_m)\})$.
Hence there exists $r>0$ with $B(\mu,r)\subset\operatorname{conv}\{\phi(y_1),\dots,\phi(y_m)\}$.

For any $\lambda\neq 0$, let $u=\lambda/\|\lambda\|$. Since $\mu+ru$ lies in
$\operatorname{conv}\{\phi(y_1),\dots,\phi(y_m)\}$, there exists some $i\in\{1,\dots,m\}$ with
\[
u\cdot\phi(y_i)\ge u\cdot\mu+r.
\]
Using this $y_i$ and the definition of $A$,
\[
A(\lambda)
\ge \lambda\cdot\phi(y_i)+\log a(y_i),
\]
and therefore
\[
f(\lambda)
\ge r\|\lambda\|+\log a(y_i)
\ge r\|\lambda\|+\min_{1\le i\le m}\log a(y_i).
\]
Thus $f(\lambda)\to+\infty$ as $\|\lambda\|\to\infty$.

\medskip\noindent\emph{Conclusion.}
Under the assumption $\Lambda=\R^d$, the function $A$ is smooth
(Remark~\ref{rem:gradient}), so $f$ is continuous. The linear lower bound shows that
for large enough $R$, the minimum of $f$ over all of $\R^d$ is achieved inside the
closed ball $\|\lambda\|\le R$; since $f$ is continuous on this compact set, it attains
its minimum there at some $\lambda^*$. Then $\nabla f(\lambda^*)=0$ gives
$\mu_{\lambda^*}=\mu$.
\end{proof}

The analogous result for exponential families on general measurable spaces is established in \cite[Chapter~3]{Brown1986}.

\section{Discussion}

This note is organized around a sharp separation between \emph{algebraic} consequences of the KL difference identity~\eqref{eq:identity} (Section~\ref{sec:algebraic}), which require nothing beyond $\KL{\cdot}{\cdot} \geq 0$, and \emph{analytic} complements (Section~\ref{sec:logpartition}), which additionally use the differentiability of $A$ on $\mathrm{int}(\Lambda)$.
The identity itself is a one-line calculation: the log-ratio $\log(p_{\lambda_1}/p_{\lambda_2})$ is affine in $\phi$, so its expectation under any $q$ involves only $A(\lambda_1)$, $A(\lambda_2)$, and $\mu_q$.

\smallskip

The purely algebraic results---the generalized three-point identity, the Pythagorean theorem, the I-projection and reverse I-projection characterizations, the convexity of $A$ and identification of $A^*$, the ELBO decomposition, and KL-regularized reward maximization---are all obtained by the same recipe: choose $q$, $\lambda_1$, $\lambda_2$ in the identity, and apply $\KL{\cdot}{\cdot} \geq 0$ with its equality characterization.

\smallskip

The analytic section adds one external input---differentiability of $A$---which enables the Bregman divergence interpretation, connects to the classical literature, and provides the gradient characterization $\nabla A(\lambda) = \mu_\lambda$ needed for the surjectivity proof:
\begin{itemize}[itemsep=3pt]
    \item The convexity of $A$ and the supporting hyperplane property are algebraic; promoting the slope $\mu_\lambda$ to a gradient requires differentiability, but a mild argument suffices: term-by-term differentiation for finite $Y$, or dominated convergence in general \cite{Brown1986,WainwrightJordan2008}.

    \item The generalized three-point identity~\eqref{eq:threepoint}, proved algebraically, is recognized as a generalization of the Bregman three-point property of \cite{Banerjee2005} once the Bregman interpretation is available.

    \item The Pythagorean theorem is classically developed using the dually flat structure of information geometry \cite{Amari1985,AmariNagaoka2000}.
    Here it is a direct consequence of the three-point identity applied to arbitrary distributions~$q$ outside the family, a generality not readily found in the standard Bregman divergence literature.
    A related Pythagorean property for I-projections was established by \cite{Csiszar1975}.

    \item Surjectivity of the moment map requires differentiability of $A$ together with a standard geometric argument about the interior of the moment space $\cM$.
\end{itemize}

\smallskip

We do not address the rich analytic structure of $A$ (higher-order cumulants, the Fisher information metric, analyticity on $\mathrm{int}(\Lambda)$), for which we refer to \cite{Brown1986,WainwrightJordan2008}.
Our aim is limited to showing how far a single algebraic identity can reach.

\smallskip

This note arose from the study of KL-regularized control problems, where the identity provides a natural framework for analyzing the geometry of optimal policies.
This connection will be developed elsewhere.

\section*{AI Disclosure}

The author used Claude (Anthropic, \texttt{claude.ai}) and ChatGPT (OpenAI) during the preparation of this manuscript for assistance with exposition, structuring arguments, and reviewing text and proof drafts. The author reviewed and edited all AI-assisted content and takes full responsibility for the content of this paper.


\end{document}